\documentclass[letterpaper, 10 pt, conference]{ieeeconf}  

\IEEEoverridecommandlockouts                              

\usepackage{amsmath}
\usepackage{amssymb}
\usepackage{graphicx}
\usepackage[Symbol]{upgreek}
\usepackage{epstopdf}
\usepackage{subcaption}
\usepackage{enumerate}
\usepackage{paralist}

\usepackage{float}

\usepackage{latexsym}
\usepackage{multicol}
\usepackage{multirow}
\usepackage{lipsum}

\usepackage{cite}

\usepackage{makecell}
\usepackage{breqn}

\usepackage{verbatim}

\usepackage{color}
\usepackage{xcolor}
\usepackage{bm}
\usepackage[font=small,labelfont=bf]{caption}

\usepackage{courier}
\usepackage{tikz}
\usetikzlibrary{calc}

\newtheorem{theorem}{Theorem}
\newtheorem{remark}{Remark}

\newcommand{\vect}[1]{\bm{#1}}

\newcommand{\bvect}[1]{\bar{\vect{#1}}}
\newcommand{\hvect}[1]{\hat{\vect{#1}}}
\newcommand{\dvect}[1]{\dot{\vect{#1}}}
\newcommand{\ddvect}[1]{\ddot{\vect{#1}}}

\def\particulartemplate#1{
  \begin{tikzpicture}[overlay, remember picture]
    \draw let \p1 = (current page.west), \p2 = (current page.east) in
      node[minimum width=\x2-\x1, minimum height=0.1cm, rectangle, fill=yellow!35!white, anchor=north west, align=center, text width=\x2-\x1] at ($(current page.north west) + (0,-0.3)$) {\large \textbf{\texttt{#1}} };
  \end{tikzpicture}
}

\title{\LARGE \bf
Human-robot collaborative object transfer using human motion prediction based on Cartesian pose Dynamic Movement Primitives
}

\author{Antonis Sidiropoulos$^{1,a}$, Yiannis Karayiannidis$^{2,b}$ and Zoe Doulgeri$^{1,a }$
\thanks{$^{1}$Aristotle University of Thessaloniki, Department of
Electrical and Computer Engineering, Thessaloniki 54124, Greece. 
        {\tt\small antosidi@ece.auth.gr, doulgeri@eng.auth.gr}}%
\thanks{$^{2}$Chalmers University of Technology, Department of Electrical Engineering, G\"oteborg, Sweeden.
        {\tt\small yiannis@chalmers.se}}%
\thanks{$^{a}$ The research leading to these results has received funding by the EU Horizon 2020 Research and Innovation Programme under grant agreement No 820767, project CoLLaboratE.}%
\thanks{$^{b}$ This research is co-financed by the Swedish Research Council (VR).}%
}

\begin{document}

\maketitle
\particulartemplate{
This paper is a Post-Print version (i.e. final draft post-refereeing). \\
The publisher's version can be accessed through \\
DOI: 10.1109/ICRA48506.2021.9562035
}
\thispagestyle{empty}
\pagestyle{empty}


\begin{abstract}
In this work, the problem of human-robot collaborative object transfer to unknown target poses is addressed. The desired pattern of the end-effector pose trajectory to a known target pose is encoded using DMPs (Dynamic Movement Primitives). During transportation of the object to new unknown targets, a DMP-based reference model and an EKF (Extended Kalman Filter) for estimating the target pose and time duration of the human's intended motion is proposed. A stability analysis of the overall scheme is provided. Experiments using a Kuka LWR4+ robot equipped with an ATI sensor at its end-effector validate its efficacy with respect to the required human effort and compare it with an admittance control scheme.
\end{abstract}


\section{Introduction} \label{sec:Introduction}

The technological advancements during the past few decades have given a lot of momentum to various research fields in robotics. Undoubtedly, physical human-robot interaction constitutes a prolific research area, with numerous applications, both in industrial and household environments \cite{HR_collab_survey}. The prospect of having robots work collaboratively with humans has attracted a lot of attention and can greatly enhance everyday life. For instance, tasks like lifting or transferring objects occur quite frequently. Robots could actively assist humans execute such tasks, mitigating their fatigue. To this end, a consensus of a common trajectory and synchronization is required to execute the task in synergy. Discord between the partners' intention  will result in high interaction forces \cite{Human_motion_mov_obj_case_study}. Hence, there is a need for intelligent proactive robot control strategies with human intention prediction in order to minimize human effort, promoting productivity and efficiency.  

For the improvement of human robot collaboration, a lot of emphasis is placed upon the prediction of the human's motion/intention.
Human motion prediction algorithms are typically based on  a 
 minimum jerk motion model 
\cite{HR_coop_manip_motion_est_2001,hi_robot_assist_2007,coord_arm_move_1985}.
An impedance model whose reference trajectory is produced by a radial basis function (RBF) network adapted online by the human exerted force, was suggested in \cite{Ge:HR_Collab_Motion_Intention_Est}. Compared to admittance, the forces required by the human were reduced, however boundedness of the reference trajectory is not proved and their results show that the estimated trajectory lagged and did not converge to the actual one. 
In \cite{Gams_Ude_2014}, Iterative Learning Control is employed to adjust the DMP's trajectory for transferring an object to a new position. However, this requires several repetitions for each new target until the task is learned.
The prediction of the goal in a high-five application with DMPs using a KF was proposed  in \cite{ProIP_2014}.
In \cite{DMP_EKF_handover} the idea of using an EKF was introduced to predict on-line the intended handover position and time based on the current position of the human hand and  a DMP that parameterizes its motion. 
Human-robot handover without involving any  prediction of the handover pose and time is considered in \cite{Sidiropoulos2019_handover}.
However,  in \cite{ProIP_2014, DMP_EKF_handover, Sidiropoulos2019_handover}, the human and the robot are not physically coupled except at the end of the motion.

In our previous work, human-robot collaborative object transfer is addressed, focusing only on Cartesian position  \cite{Antosidi_HRCOT_2019}.
In this paper, we extend our method to include not only the position but also the target orientation. 
The robot is only aware of the pattern of motion, but is agnostic to the target pose and how fast the movement should be executed (time scaling). The objective is to render the robot proactive in the execution of the learned motion pattern scaled to unknown targets and time scales by anticipating humans' intention, thus minimizing their effort.
To the best of our knowledge all previous works on predicting human motion focus solely on the Cartesian position. A possible reason is that any parameterization of orientation is non-linear which complicates the formulation and design of the controller and observer. 
Hence, any extension including the prediction of the target orientation and its real-time use in the control loop  is far from trivial and it can  prove to be quite challenging. In this work, a modification of the typical DMP orientation formulation was required in order to solve the control problem.
The main contribution of this work is  therefore 
a combined control and prediction scheme that includes both target orientation and position and its stability proof. 
The proposed scheme's  practical efficiency is demonstrated in a number of experiments.
 
 


\section{Proposed approach} \label{sec:Proposed_approach}

Our focus is on tasks, where the motion pattern is essentially the same, but the initial/target poses and time duration of the movement can be different leading to the spatio-temporal scaling of the motion pattern. This occurs for instance in  box stacking tasks.
We make use of DMPs to encode the Cartesian pose of the robot's end-effector during a point to point motion, that is recorded from a kinesthetic demonstration. This recorded motion pattern is employed by the robot for collaborative transportation of an object to different targets and at different time scalings unknown to the robot.
Human-Human collaborative object transfer studies reveal that object motion is highly correlated with a motion pattern \cite{Fi_collab_obj_manip,Human_motion_mov_obj_case_study}.

For position, let the motion to a new target $\vect{p}_g \in \mathbb{R}^3$ with time scaling $\tau_p$, be generated by the position DMP model:
\begin{equation} \label{eq:DMP_ddpos}
    \ddvect{p} = \frac{1}{\tau_p^2}(\alpha_z \beta_z (\vect{p}_g - \vect{p}) - \alpha_z \tau_p \dvect{p} + g_f(t/\tau_p)\vect{K}_{p_g}\vect{f}_p(t/\tau_p))
\end{equation}
where $\vect{p}, \dvect{p}, \ddvect{p} \in \mathbb{R}^3$ are the position, velocity and acceleration of the robot's end-effector and the rest terms are detailed in Appendix B. 
Accordingly for the orientation, the motion to a new target orientation $\vect{Q}_g \in \mathbb{S}^3$ with time scaling $\tau_o$ is generated  by the following orientation DMP formulation:
\begin{equation} \label{eq:DMP_ddq_1}
    \ddvect{q}' = \frac{1}{\tau_o^2}(\alpha_z \beta_z(\vect{q}'_{g} - \vect{q}') - \alpha_z \tau_o \dvect{q}' + g_f(t/\tau_o)\vect{K}_{q_g}\vect{f}_o(t/\tau_o))
\end{equation}
where $\vect{q}' = \log(\vect{Q}*\bvect{Q}_0) \in \mathbb{R}^3$ with $\vect{Q}$, $ \vect{Q}_0 \in \mathbb{S}^3$ being the current and initial orientation and $\bvect{Q}_0$ the conjugate of $\vect{Q}_0$, $\log(.)$ the quaternion logarithm and $\vect{q}'_g = \log(\vect{Q}_g*\bvect{Q}_0)$, Unit quaternion preliminaries can be found in Appendix A, and more details on \eqref{eq:DMP_ddq_1} in Appendix B.

In the above DMP models, the forcing terms $\vect{f}_p, \vect{f}_o$ comprise of a weighted sum of Gaussians which encode the motion pattern. The matrices $\vect{K}_{p_g}$, $\vect{K}_{q_g}$ provide the spatial scaling of the motion based on the target and the gating function  $g_f(.)$ ensures that the forcing term vanishes at the end of the motion (Appendix B). 

\begin{remark} \label{remark:DMP_orient}
Writing the DMP model with the orientation anchored to the initial orientation $\vect{Q}_0$ is important as it decouples $\vect{q}'$ from $\vect{Q}_g$ which greatly simplifies the orientation DMP formulation in case time varying target estimates are to be utilized. This is in contrast to $\vect{q}'=\log(\vect{Q}_g*\bvect{Q})$ from \cite{DMP_orient_Koutras}, where the orientation is anchored to the target orientation $\vect{Q}_g$.
This formulation
is exploited in the design of the reference model detailed in the next Section.
\end{remark}

\begin{remark} \label{remark:time_scalings}
The use of separate time scaling variables for position and orientation, $\tau_p, \tau_o$, is adopted. This provides greater versatility and generalization, since it accommodates also the scenario of having position and orientation patterns with different temporal duration during execution. 
\end{remark}

During the collaborative object transfer we assume the robot grasps the object rigidly and compensates its weight, while a F/T sensor is mounted at its end-effector. The external force/torque exerted by the human is calculated by subtracting from the F/T sensor's measurement the object's wrench. 
It is assumed that no other contacts with the environment occur, as this would require additional sensors to discriminate contact forces from the human intended forces.
A velocity controlled robot is considered, implying that any reference velocity can be accurately tracked.
The core idea is to design a control scheme that will: 1) render the robot compliant to the forces/torques exerted by the human so that he can move the object along the trajectory intended by him, 2) make the robot proactive in the tracking of the intended trajectory, by estimating online the target and time-scaling, in order to minimize the effort required by the human. To achieve this objective we propose the control scheme depicted in Fig. \ref{fig:control_scheme} where:
\begin{itemize}
    \item The \textit{DMP based reference model} takes as input the external force/torque exerted by the human, and using the current state $\vect{p},\dvect{p}, \vect{Q}, \vect{\omega}$ (pose, velocity) and the estimates of the target pose and time scalings, $\hvect{p}_g, \hvect{Q}_g, \hat{\tau}_p, \hat{\tau}_o$ generates a DMP-model based estimated trajectory, that is shaped by the external force/torque.
    \item The \textit{DMP based EKF} takes as input the force/torque exerted by the human and the current state to produce estimates of the target pose and time-scaling.
\end{itemize}

In the subsequent sections we dwell on the formulation and design of the reference model and the observer.

\begin{figure}[!ht]
	\centering
		\includegraphics[scale=0.19]{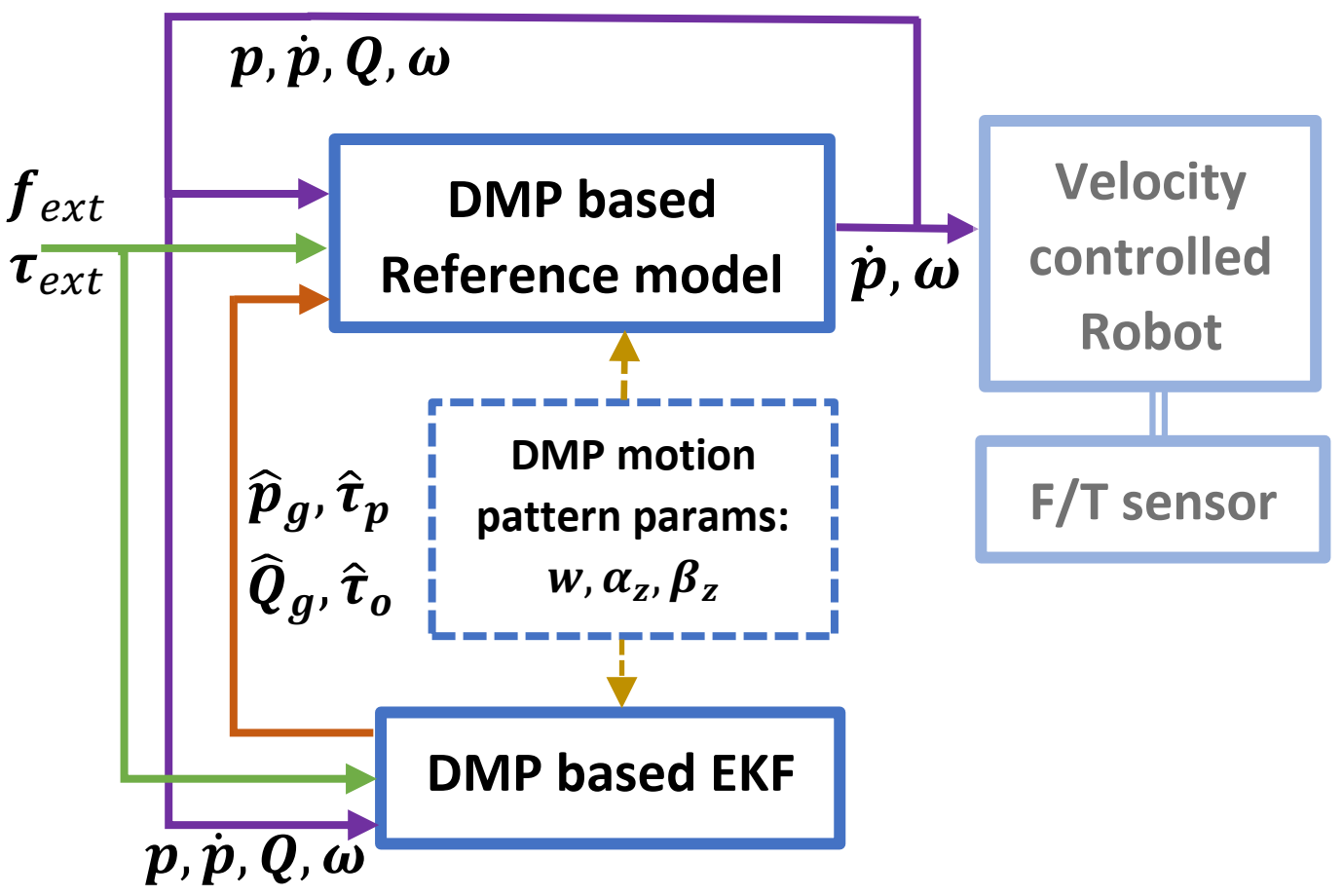}
	\caption{Proposed approach.}
	\label{fig:control_scheme}
\end{figure}


\section{Reference model} \label{sec:Ref_model}

The reference model for position is formulated as:
\begin{equation} \label{eq:robot_model_ref_p}
    \vect{M}_p \ddvect{p} = \vect{M}_p\hat{\ddvect{p}} + \vect{f}_{ext}
\end{equation}
\begin{equation} \label{eq:DMP_ddp_hat}
    \hat{\ddvect{p}} = \frac{1}{\hat{\tau}_p^2}(\alpha_z \beta_z (\hvect{p}_g - \vect{p}) - \alpha_z \hat{\tau}_p \dvect{p} + g_f(t/\hat{\tau}_p)\hvect{K}_{p_g}\vect{f}_p(t/\hat{\tau}_p)) \nonumber
\end{equation}
where $\vect{f}_{ext}$ is the force exerted by the human, $\vect{M}_p = diag(m_x,m_y,m_z)>0$ is the reference model's inertia which weights the effect of $\vect{f}_{ext}$ on the reference model's trajectory
and $\hat{\ddvect{p}}$ is the estimated acceleration based on the position DMP model \eqref{eq:DMP_ddpos}, using the estimated target $\hvect{p}_g$ and time scaling $\hat{\tau}_p$.

Similar to position, the reference model for orientation is formulated as:
\begin{equation} \label{eq:robot_model_ref_o}
    \vect{M}_o \dvect{\omega} = \vect{M}_o\hat{\dvect{\omega}} + \vect{\tau}_{ext}
\end{equation}
\begin{equation} \label{eq:DMP_dvRot_hat}
    \hat{\dvect{\omega}} = vec\{2 (\vect{J}_{q'}\dvect{q}' + \vect{J}_{q'}\hat{\ddvect{q}}')*\bvect{Q}'  - \frac{1}{2}\begin{bmatrix} ||\vect{\omega}||^2 \\ \vect{0}_{3 \times 1} \end{bmatrix} \}
\end{equation}
with
\begin{equation} \label{eq:DMP_ddq_hat}
     \hat{\ddvect{q}}' = \frac{1}{\hat{\tau}_o^2}(\alpha_z \beta_z(\hvect{q}'_{g} - \vect{q}') - \alpha_z \hat{\tau_o} \dvect{q}' + g_f(t/\hat{\tau}_o)\hvect{K}_{q_g}\vect{f}_o(t/\hat{\tau}_o))
\end{equation}
where $\vect{M}_o = diag(I_x,I_y,I_z)>0$ is the reference model's inertia, $\vect{\tau}_{ext}$ is the torque exerted by the human and $\hat{\dvect{\omega}}$ is the estimated rotational acceleration corresponding to the estimated acceleration $\hat{\ddvect{q}}'$ (see \eqref{eq:dvRot_ddq} from Appendix A). Finally $vec\{.\}$ denotes the vector part of a quaternion.
Equation \eqref{eq:DMP_ddq_hat} is derived from \eqref{eq:DMP_ddq_1} 
using the estimated target $\hvect{q}'_g = \log(\hvect{Q}_g*\bvect{Q}_0)$
and time scaling $\hat{\tau}_o$. 

\begin{remark}
Further to remark \ref{remark:DMP_orient}, notice that in \eqref{eq:robot_model_ref_o}, \eqref{eq:DMP_dvRot_hat}, if
$\vect{q}'=\log(\vect{Q}_g*\bvect{Q})$ was used instead of $\vect{q}'=\log(\vect{Q}*\bvect{Q}_0)$
then apart from $\hvect{q}'_g$, all other terms would also be dependent on $\hvect{Q}_g$ (and some of them on its derivative too). This would complicate the design of the reference model and  the results of the stability analysis in Section \ref{sec:Stability_analysis} would most likely be inconclusive.
\end{remark}



\section{DMP-based EKF observer for target and time scale prediction} \label{sec:Proposed_EKF}

We assume that the trajectory generated from the reference model \eqref{eq:robot_model_ref_p}, \eqref{eq:robot_model_ref_o} corresponds to the trajectory produced by the DMP models  \eqref{eq:DMP_ddpos}, \eqref{eq:DMP_ddq_1} for constant target $\vect{p}_g, \vect{Q}_g$ and time-scalings $\tau_p,\tau_o$, which we want to estimate. 
This assumption is reasonable because, as highlighted in Section \ref{sec:Proposed_approach}, we focus on tasks where the motion pattern is essentially the same, but the initial/target poses and time duration of the movement can be different leading to the spatio-temporal scaling of the motion pattern.
To this end, we construct a DMP-based EKF with states $\vect{\theta}_p = \left[\vect{p}_g^T \ \tau_p \right]^T$, $\vect{\theta}_o = \left[{\vect{q}'_g}^{T} \ \tau_o \right]^T$ with the following state and measurement equations for $i \in \{p,o\}$,:
\begin{equation} \label{eq:state_trans_fun}
    \dvect{\theta}_i = \vect{0}_{4 \times 1} 
\end{equation}
\begin{equation} \label{eq:msr_fun}
    \vect{z}_i = \vect{h}_i(\vect{\theta}_i, \vect{s}_i, t) 
\end{equation}
 where $\vect{z}_p = \ddvect{p}$, $\vect{z}_o = \dvect{\omega}$, $\vect{s}_p = [\vect{p}^T \ \dvect{p}^T]^T$, $\vect{s}_o = [\vect{Q}^{'T} \ \vect{\omega}^T]^T$ and:
\begin{equation*} \label{eq:DMP_hp}
    \vect{h}_p(\vect{\theta}_p, \vect{s}_p, t) \triangleq 
    \frac{1}{\tau_p^2}(\alpha_z \beta_z (\vect{p}_g - \vect{p}) - \alpha_z \tau_p \dvect{p} + g_f(\frac{t}{\tau_p})\vect{K}_{p_g}\vect{f}_p(\frac{t}{\tau_p})) 
\end{equation*}
\begin{equation*} \label{eq:DMP_ho}
    \vect{h}_o(\vect{\theta}_o, \vect{s}_o, t) \triangleq 
    vec\{2 (\vect{J}_{q'}\dvect{q}' + \vect{J}_{q'}\ddvect{q}')*\bvect{Q}'\}
\end{equation*}
with $\ddvect{q}'$ given by \eqref{eq:DMP_ddq_1}.

For each system $i \in \{p,o\}$ in \eqref{eq:state_trans_fun}-\eqref{eq:msr_fun} we construct a fading memory EKF observer \cite{1998_modified_EKF, Simon_KF} with projection \cite{Simon_KF} and normalization \cite{Robust_adapt_control_Ioannou}:
\begin{equation} \label{eq:observer}
    \dot{\hvect{\theta}}_i = \vect{K}_i ( \vect{z}_i -  \hat{\vect{z}}_i )
\end{equation}
\begin{equation} \label{eq:measure_est}
    \hvect{z}_i = \vect{h}_i(\hvect{\theta}_i, \vect{s}_i, t)
\end{equation}
where $\hvect{\theta}_i$ is the state estimate of the $i_{th}$ observer and $\vect{K}_i \in \mathbb{R}^{4 \times 3}$ is a time varying gain matrix given by:
\begin{equation} \label{eq:observer_gain}
    \vect{K}_i = \vect{N}_i(t)\vect{P}_i(t) \bvect{C}_i^T(t) \vect{R}_i^{-1}
\end{equation}
where $\vect{N}_i(t)$ is a projection matrix which ensures the estimates respect certain bounds, $\vect{P}_i(t)$ is given by the solution of the the following equation
\cite{1998_modified_EKF}, \cite{Simon_KF}:
\begin{equation} \label{eq:P_saturate}
    \dot{\vect{P}}_i = 
    \begin{cases}
        2 a_p \vect{P}_i + (\vect{K}_i\vect{R}_i-\vect{P}_i\bvect{C}_i^T)\vect{R}_i^{-1}(\vect{K}_i\vect{R}_i-\vect{P}_i\bvect{C}_i^T)^T
        \\  + \vect{Q}_i - \vect{P}_i \bvect{C}_i^T(t) \vect{R}_i^{-1} \bvect{C}_i \vect{P}_i   \quad ,  \ ||\vect{P}_i|| \le \rho_2 \\
        \vect{0} \ , \ otherwise
    \end{cases}
\end{equation}
$a_p > 0$, $\vect{R}_i$, $\vect{Q}_i$ are the measurement and process noise covariance matrices and $\bar{\vect{C}}_i(t)$ the linearized measurement equation matrix with normalization found by $\bar{\vect{C}}_i(t) = \vect{C}_i(t) / c_{n,i}$, 
where $\vect{C}_i(t) = \left. \frac{\partial \vect{h}_i(\vect{\theta}_i, \vect{u}_i, t) }{\partial \vect{\theta}_i} \right|_{\vect{\theta}_i=\hvect{\theta}_i }$ and $c_{n,i} = \sqrt{1 + \lambda_{max}(\vect{C}_i(t)\vect{C}_i^T(t))}$. Normalization is used to ensure the boundedness of the observer's update law, irrespective of the boundedness of $\vect{s}_i$ \cite{Robust_adapt_control_Ioannou}.  
Notice that $\vect{C}_i(t) = \left. \frac{\partial \vect{h}_i(\vect{\theta}_i, \vect{u}_i, t) }{\partial \vect{\theta}_i} \right|_{\vect{\theta}_i=\hvect{\theta}_i }$ can be computed analytically, similarly to \cite{DMP_EKF_handover}, but its analytical expression is omitted here due to space restrictions.
It is worth mentioning that if $\log(\vect{Q}_g*\bvect{Q})$ was used instead of $\vect{q}'_g$, then the analytic derivation becomes practically intractable and one would have to resort to numerical differentiation, increasing the observer's computational complexity and possibly compromising its performance.
The design constant $\rho_2>0$ is used to ensure that the covariance matrix remains bounded.
The projection matrix $\vect{N}_i$ is derived from the constraints $\vect{D}_i \vect{\theta}_i \le \vect{d}_i$ based on a least squares approach \cite{Simon_KF},
where $\vect{d}_i = [\bar{\vect{\theta}}_i^T \ -\underline{\vect{\theta}}_i^T]^T$ and $\vect{D}_i = \begin{bmatrix} \ \vect{I}_{4} \ -\vect{I}_{4} \end{bmatrix}^T$. Bounds $\bar{\vect{\theta}}_i, \ \underline{\vect{\theta}}_i$ stem from the physical interpretation of the estimated parameters, i.e. the target position is constrained by the robot's workspace. Moreover, the time scalings are positive and finite. 
The projections matrix $\vect{N}_i$ is then given by \cite{Simon_KF}:
\begin{equation} \label{eq:proj_mat}
    \vect{N}_i = 
    \left( \vect{I}_{4} - \bar{\vect{D}}_i^T (\bar{\vect{D}_i} \bar{\vect{D}}_i^T)^{-1} \bar{\vect{D}}_i \right)
\end{equation}
with the active constraints satisfying $\bar{\vect{D}}_i \hat{\vect{\theta}}_i = \bar{\vect{d}}_i$ and $\bar{\vect{D}}_{i,j} \dot{\hvect{\theta}}_i > \vect{0}$, where $\bar{\vect{D}}_{i,j}$ is the $j_{th}$ row of $\bar{\vect{D}}_i$ and $\bar{\vect{D}}_i, \bar{\vect{d}}_i$ are subset of the rows of $\vect{D}_i, \vect{d}_i$ respectively.

Finally, the measurement error can be obtained from \eqref{eq:robot_model_ref_p}, \eqref{eq:robot_model_ref_o}, i.e. $\vect{z}_p - \hvect{z}_p = \ddvect{p} - \hat{\ddvect{p}} = \vect{M}_p^{-1}\vect{f}_{ext}$ and $\vect{z}_o - \hvect{z}_o = \dvect{\omega} - \hat{\dvect{\omega}} = \vect{M}_o^{-1}\vect{\tau}_{ext}$. Therefore, for the update law (\ref{eq:observer}) we can utilize the following expression:
\begin{equation} \label{eq:dtheta_hat}
    \dot{\hvect{\theta}}_i = \vect{K}_i \vect{M}_i^{-1}\frac{\vect{\nu}_i}{c_{n,i}}
\end{equation}
where $\vect{\nu}_p = \vect{f}_{ext}$ and $\vect{\nu}_o = \vect{\tau}_{ext}$.


\begin{remark}
For an observer of the form \eqref{eq:observer}, the covariance matrix $\vect{P}$ satisfies \cite{Murray_OptBasedCtrl} $\dvect{P} = (\vect{K}\vect{R} - \vect{P}\vect{C}^T)\vect{R}^{-1}(\vect{K}\vect{R} - \vect{P}\vect{C}^T)^T 
    + \vect{Q} - \vect{P}\vect{C}^T\vect{R}\vect{C}\vect{P}$
for any matrix $\vect{K}$. 
In the case of the fading memory filter \cite{Simon_KF}, \cite{1998_modified_EKF} the term $a_p\vect{P}$, $a_p>0$ is added to $\dvect{P}$. 
The Kalman gain is derived by minimizing $\dvect{P}$ w.r.t. $\vect{K}$ which yields $\vect{K} = \vect{P}\vect{C}^T\vect{R}^{-1}$.
When projection is employed, the Kalman gain is further modified by the projection matrix $\vect{N}$ \cite{Simon_KF}, resulting in \eqref{eq:observer_gain} which in turn yields the covariance matrix update \eqref{eq:P_saturate}.
\end{remark}

\begin{remark}
We chose specifically the EKF with fading memory since is has been shown to be robust against non-linearities (see \cite{1998_modified_EKF} and Chapter $7.4$ from \cite{Simon_KF}). 
Moreover, we compared it to the UKF (Unscented Kalman Filter) and we found the EKF to perform slightly better in our case, which has also been observed in other applications \cite{EKF_vs_UKF_quatMotion2003,EKF_vs_UKF_gps2004,EKF_vs_UKF_fnn2012}.
\end{remark}

\section{Stability Analysis} \label{sec:Stability_analysis}

Due to the nonlinearity of the motion pattern with respect to the target and time scaling, the estimates of the EKF are guaranteed to converge to the actual ones locally, as studied in \cite{DMP_EKF_handover}. Thanks to the projection however the estimates remain bounded. However, this does not imply that the reference model's state (hence the robot's state) will remain bounded, which is crucial. Concerning this, the following theorem can be proven regarding the reference model:

\begin{theorem} \label{theo:stability}
   The model reference given by \eqref{eq:robot_model_ref_p}, \eqref{eq:robot_model_ref_o} along with the observer based on the EKF given by \eqref{eq:dtheta_hat} ensures that $\vect{p},\dvect{p}, \vect{Q}, \vect{\omega} \in \vect{L}_{\infty}$ if $\vect{f}_{ext},\vect{\tau}_{ext} \in \vect{L}_{\infty} \cap \vect{L}_2$.
\end{theorem}

\begin{proof}
The proof for the boundedness of $\vect{p}, \dvect{p}$ given that $\vect{f}_{ext} \in L_2 \cap L_{\infty}$ is identical to the proof in \cite{Antosidi_HRCOT_2019}. The orientation is also bounded since $\vect{Q} \in \mathbb{S}^3$. For the rotational velocity from \eqref{eq:robot_model_ref_o}, \eqref{eq:DMP_dvRot_hat}
we get:
\begin{equation*}
    \dvect{\omega} =  vec\{2 (\vect{J}_{q'}\dvect{q}' + \vect{J}_{q'}\hat{\ddvect{q}}')*\bvect{Q}' - \frac{1}{2}\begin{bmatrix} ||\vect{\omega}||^2 \\ \vect{0}_{3 \times 1} \end{bmatrix} \} + \vect{M}_o^{-1}\vect{\tau}_{ext}
\end{equation*}
Substituting $\dvect{\omega}$ from \eqref{eq:dvRot_ddq}, replacing  $\vect{Q}$ by $\vect{Q}'$ and $\vect{q}$ by $\vect{q}'$ we can arrive after some mathematical calculations at:
\begin{align*}
    &\vect{J}_{q'}(\ddvect{q}' - \hat{\ddvect{q}}') = \frac{1}{2}\vect{T}_{ext}*\vect{Q}'
\end{align*}
where $\vect{T}_{ext} = [ 0 \ (\vect{M}_o^{-1}\vect{\tau}_{ext})^T ]^T$. Multiplying both sides by $\vect{J}_{Q'}$ and using the identity $\vect{J}_{Q'}\vect{J}_{q'} = \vect{I}_3$ \cite{DMP_orient_Koutras}, we arrive at:
\begin{equation*}
    \ddvect{q}' - \hat{\ddvect{q}}' = \frac{1}{2}\vect{J}_{Q'}(\vect{T}_{ext}*\vect{Q}')
\end{equation*}
Substituting $\hat{\ddvect{q}}'$ from  \eqref{eq:DMP_ddq_hat} we obtain:
\begin{equation} \label{eq:model_ref_rewritten}
    \hat{\tau}_o^2\ddvect{q}' =  -\alpha_z \beta_z \vect{q}' - \alpha_z \hat{\tau}_o \dvect{q}' + \hat{\tau}_o^2\vect{d}
\end{equation}
where $\vect{d} = \frac{1}{\hat{\tau}_o^2} ( \alpha_z \beta_z \hvect{q}'_g + g_f(\hat{x})\hvect{K}_{q_g}\vect{f}_o(\hat{x}) + \frac{1}{2}\vect{J}_{Q'}(\vect{T}_{ext}*\vect{Q}'))$ can be viewed as a time varying bounded disturbance since $\hat{\tau}_o, \hvect{q}'_g, \vect{T}_{ext}, \vect{Q}', \vect{J}_{Q'}, \vect{f}_o(\hat{x}) \in \vect{L}_{\infty}$.
It is important to highlight that the derivation of  (15) was made possible by the modified orientation DMP introduced in this work (Appendix B). 
Thus, for the rest of the proof we can follow  \cite{Antosidi_HRCOT_2019}, to show that $\vect{q}', \dvect{q}' \in \vect{L}_{\infty}$ hence $\vect{q}, \dvect{q} \in \vect{L}_{\infty}$, therefore from \eqref{eq:omega_dq} we conclude that $\vect{\omega} \in \vect{L}_{\infty}$.

The proof for the boundedness of $\vect{q}, \dvect{q}$ is provided for completeness in Appendix C.
\end{proof}
Notice that $\vect{f}_{ext}, \vect{\tau}_{ext} \rightarrow \vect{0}$, \eqref{eq:dtheta_hat} implies $\dot{\hvect{\theta}}_i \rightarrow 0$, 
and from \eqref{eq:robot_model_ref_p}, \eqref{eq:robot_model_ref_o}  $\vect{p} \rightarrow \hvect{p}_g$, $\dvect{p} \rightarrow \vect{0}$, $\vect{Q} \rightarrow \hvect{Q}_g$, $\vect{\omega} \rightarrow \vect{0}$. Considering that the condition $\vect{f}_{ext}, \vect{\tau}_{ext} \rightarrow \vect{0}$ is met iff the human has reached the desired target pose, this further implies that $\hvect{p}_g \rightarrow \vect{p}_g$, $\hvect{Q}_g \rightarrow \vect{Q}_g$. 

\begin{remark}
Theorem \ref{theo:stability} holds for any observer combined with the reference model \eqref{eq:robot_model_ref_p}, \eqref{eq:robot_model_ref_o}  as long as it produces bounded estimates and the estimation update law is bounded and has bounded energy, i.e. $\hvect{\theta} \in L_{\infty}$ and $\dot{\hvect{\theta}} \in L_{\infty} \cap L_2$.
\end{remark}


\section{Experimental results} \label{sec:Experimental_results}
The experimental setup consists of a Kuka LWR4+ robot equipped with an ATI F/T sensor at its wrist. A rectangular long box is mounted at the robot's wrist as shown in Fig. \ref{fig:exp_setup}.
The object's dynamics were identified offline and were compensated during the experiments.
The training phase involved a single demonstration  by kinesthetically guiding the robot holding the box from the other side, and training a DMP for position and orientation. During the collaborative object transfer to new targets from different initial poses the robot is  under velocity control and is driven by the reference velocity produced by the proposed approach DMP+EKF \eqref{eq:robot_model_ref_p}, \eqref{eq:robot_model_ref_o}. We have further implemented an admittance controller for comparison. The control cycle was set to $2ms$. 
The parameters chosen for the DMP are $N=30$, $\alpha_z=40$, $\beta_z=10$, for the reference model $\vect{M}_p=2\vect{I}_3$, $\vect{M}_o=0.1\vect{I}_3$, and for the observer $\vect{P}_{p}(0) = \vect{P}_{o}(0) = diag(1,1,1,10)$, $\vect{R}_p = \vect{R}_o = 2000\vect{I}_3$, $\vect{Q}_p = \vect{Q}_o = 0.001\vect{I}_4$, $a_p=1.001$, $\rho_2 = 10000$, $\bar{\vect{\theta}}_p = [0.75 \ 0.7 \ 0.95 \ 60.0]^T$, $\underline{\vect{\theta}}_p = [-0.75 \ -0.7 \ -0.2 \ 1.0]^T$, $\bar{\vect{\theta}}_o = [2\pi \ 2\pi \ 2\pi \ 60.0]^T$, $\underline{\vect{\theta}}_o = [-2\pi \ -2\pi \ -2\pi \ 1.0]^T$. For the EKF the discrete implementation was employed \cite{Simon_KF}. The admittance model  $\vect{M}\dvect{V} + \vect{D}\vect{V} = \vect{F}_{ext}$, with $\vect{V}=[\dvect{p}^T \ \vect{\omega}^T]^T$ and $\vect{F}_{ext}=[\vect{f}_{ext}^T \ \vect{\tau}_{ext}^T]^T$
was utilized with the following  parameters, tuned manually for best performance: $\vect{M}=diag(1.3\vect{I}_3, 0.08\vect{I}_3)$ and $\vect{D}=diag(25\vect{I}_3,0.6\vect{I}_3)$.
In all scenarios, the desired target pose is about $1$ cm from the placing surface, so that no contact forces emerge.

Results for three different collaborative object transfers are shown in Fig. \ref{fig:exp3_est}. In all cases the initial target estimate was set equal to the robot's initial pose. 
The initial time scaling estimate was set to $6$, while the demonstrations duration was $4.7$ sec. Thus we make a more conservative initial estimate of the time scaling by assuming that the motion will be executed slower than in the demonstration.
The DMP clock and the estimation process commence when $||\vect{F}_{ext}||>1$ N, to synchronize the robot with the human interaction.
Notice that the estimates converge to a small region around the actual target either before the end of the motion (e.g. $\hvect{p}_{g,y}$ and $\hvect{q}'_x$ in experiment $1$) or by  evolving towards the right direction (e.g. $\hvect{p}_{g,y}$ and $\hvect{q}'_z$ in experiment $2$). This is crucial as it contributes considerably to the reduction of the human's effort, as can be seen in Fig. \ref{fig:exp3_power}, where the absolute power required by the user using DMP+EKF (blue graphs) is compared to using admittance (red graphs). 
Concerning the times scaling, increase in the estimates means that the estimator infers that the motion is to be executed slower, while decrease implies the converse. Moreover, all time scaling estimates converge ultimately to a steady state value.

We have further conducted experiments of a collaborative object transfer to a new target with 5 users. Each user repeated the experiment 5 times to extract some statistical measures. In all cases the initial target estimate was set equal to the robot's initial pose. In Fig. \ref{fig:multi_power} the average power $\pm$ the standard deviation over five repetitions are depicted for each user. 
Using DMP+EKF the total absolute work for all users varied within $[1.28, \ 2.11] \ J$, the mean $L_2$ norm of the force $[0.77, \ 1.17] \ N$ and the mean $L_2$ norm of the torque $[0.13, \ 0.17] \ Nm$. In contrast, with admittance, the work varied within $[7.1, \ 11.9] \ J$, the mean $L_2$ norm of the force $[4.26, \ 6.36] \ N$ and the mean $L_2$ norm of the torque $[0.42, \ 0.71] \ Nm$.




\begin{figure}[ht]
	\centering
	\includegraphics[width=0.4\textwidth,height=0.22\textwidth]{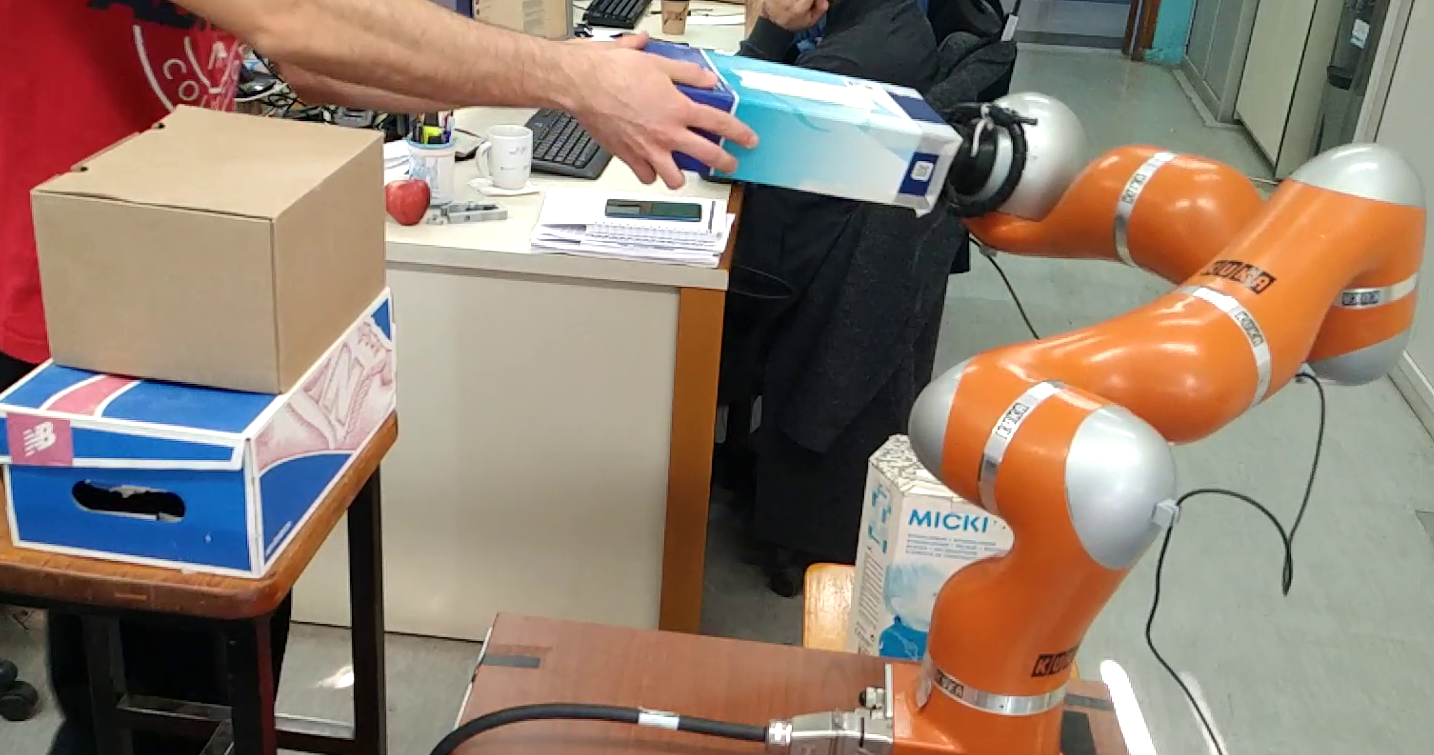}
	\caption{Experimental setup.}
	\label{fig:exp_setup}
\end{figure}

\begin{figure*}[!htbp]
    \centering
    \begin{subfigure}[b]{0.33\textwidth}
        \centering
        \includegraphics[scale=0.38]{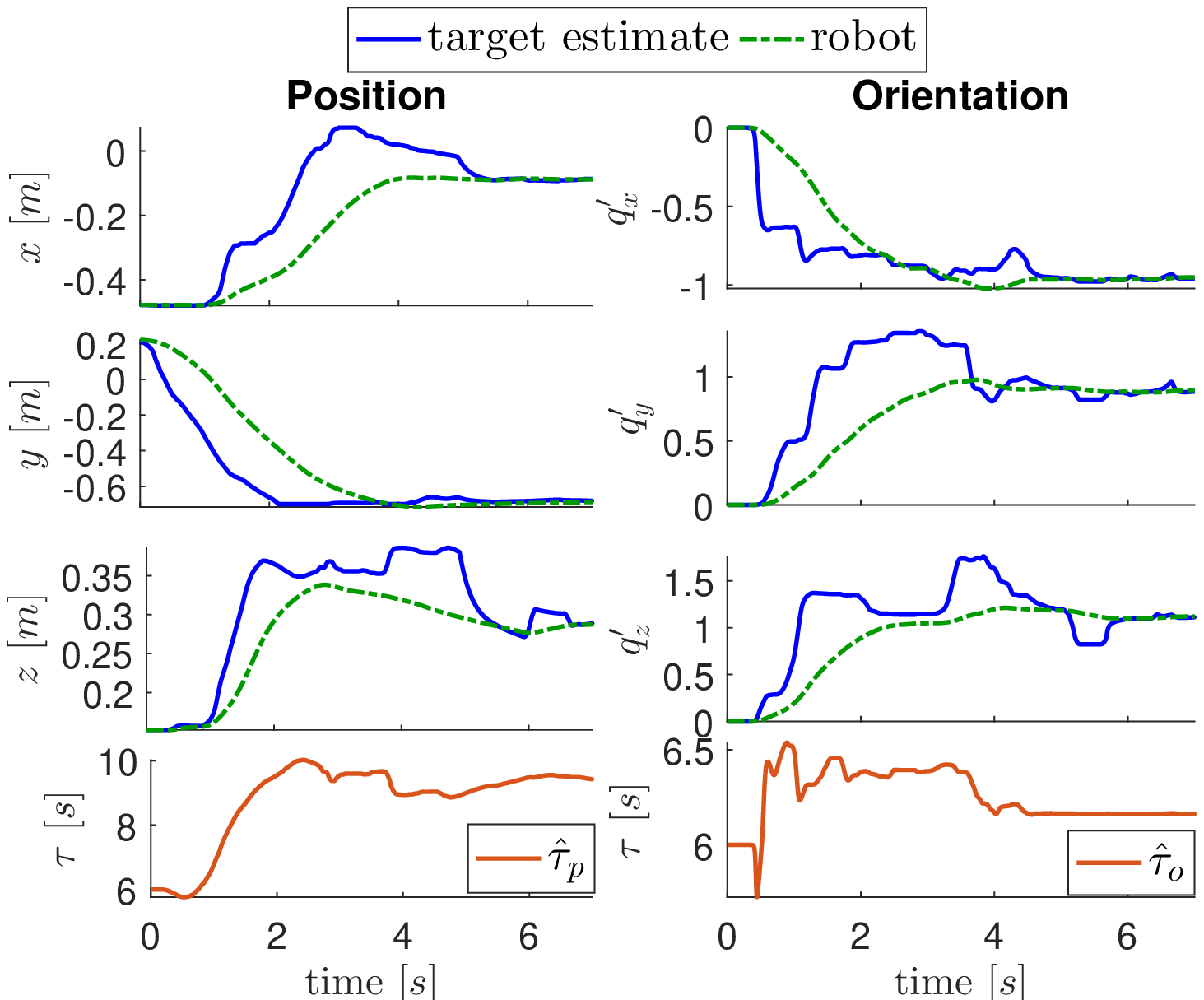}
        \caption{experiment 1}
    \end{subfigure}%
    \begin{subfigure}[b]{0.33\textwidth}
        \centering
        \includegraphics[scale=0.38]{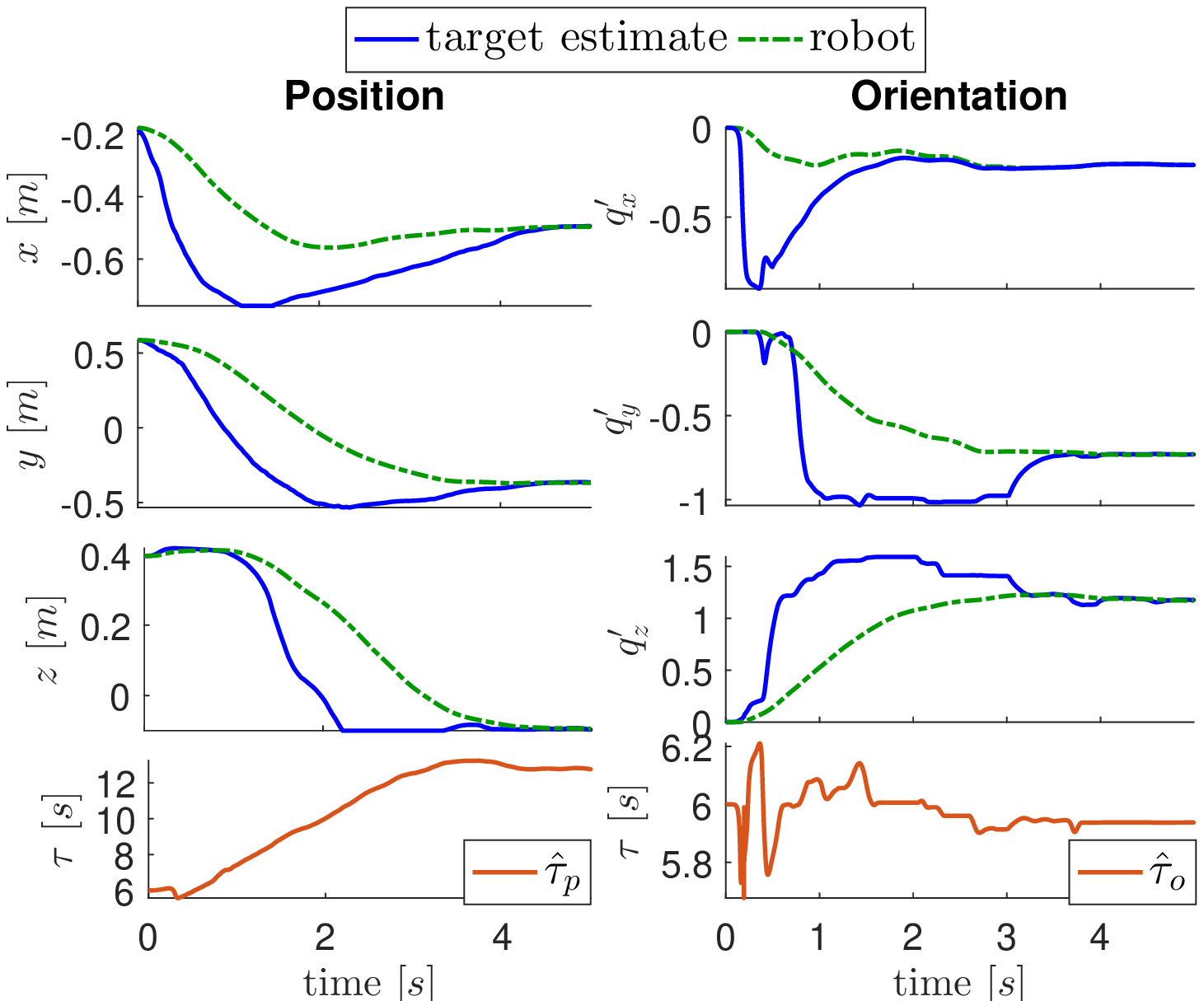}
        \caption{experiment 2}
    \end{subfigure}
    \begin{subfigure}[b]{0.33\textwidth}
        \centering
        \includegraphics[scale=0.38]{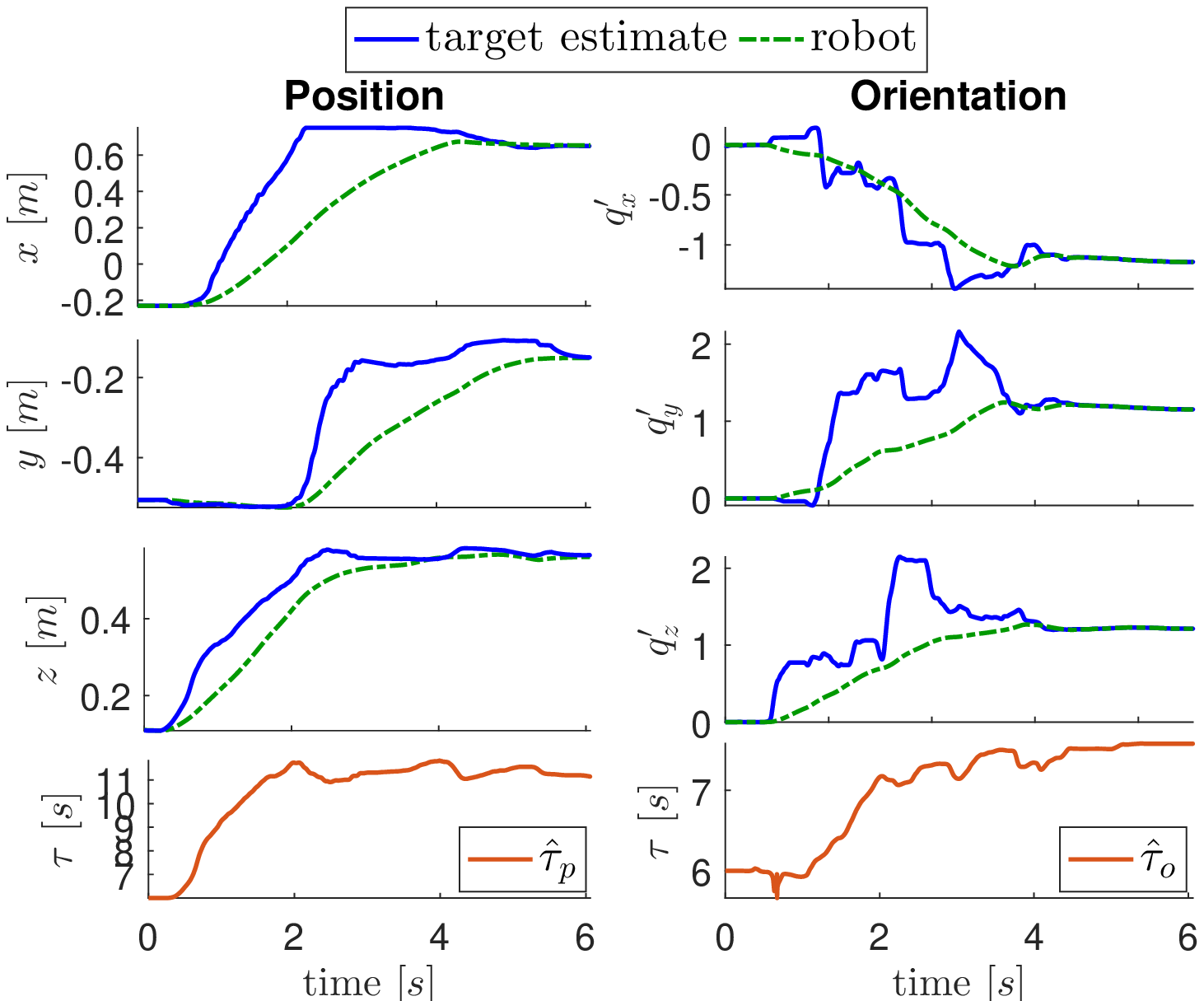}
        \caption{experiment 3}
    \end{subfigure}%
    \caption{Estimation results for the three first object transfer experiments. The estimates of the target position and orientation are plotted with solid blue lines and the corresponding robot's position/orientation with dash-dotted green line. The time scaling estimates are plotted with light brown lines. For orientation all quantities are expressed as the quaternion logarithm w.r.t the initial orientation $\vect{Q}_0$}
    \label{fig:exp3_est}
\end{figure*}

\begin{figure*}[!htbp]
    \centering
    \begin{subfigure}[b]{0.33\textwidth}
        \centering
        \includegraphics[scale=0.4]{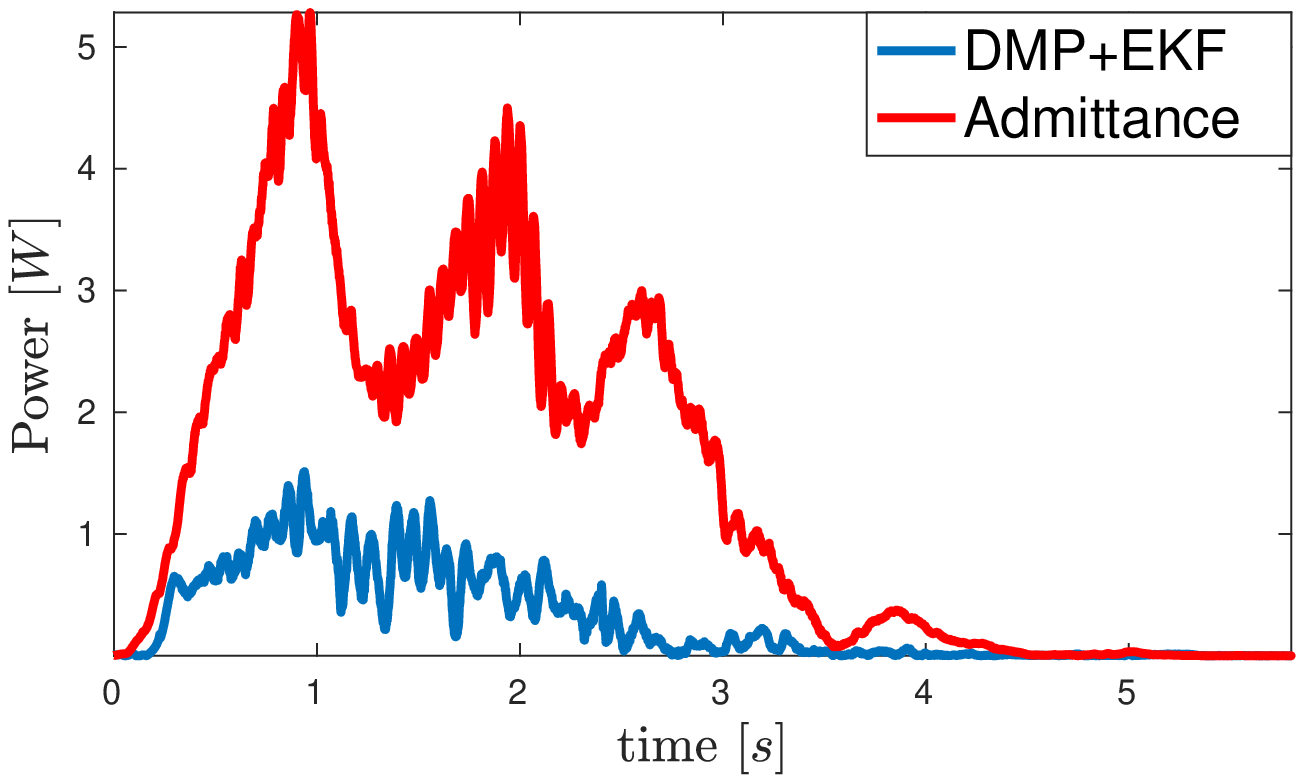}
        \caption{experiment 1}
    \end{subfigure}%
    \begin{subfigure}[b]{0.33\textwidth}
        \centering
        \includegraphics[scale=0.4]{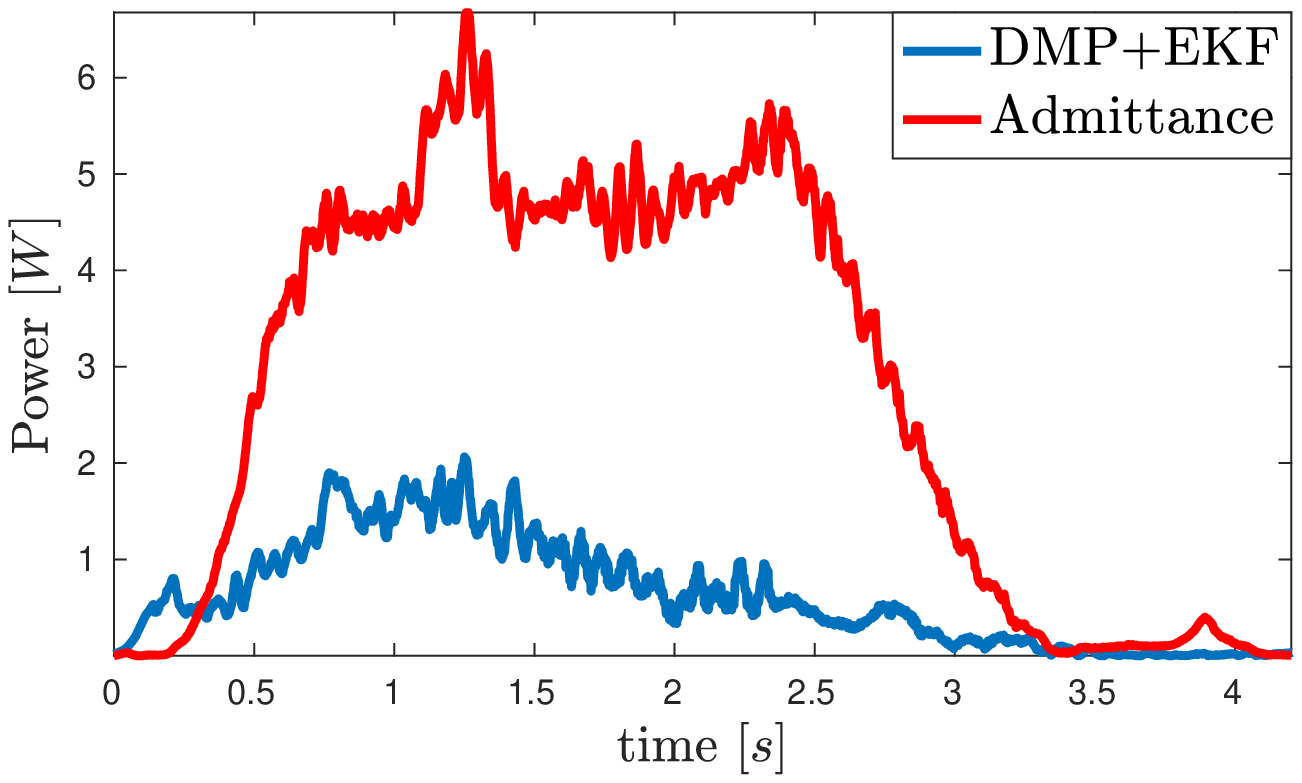}
        \caption{experiment 2}
    \end{subfigure}
    \begin{subfigure}[b]{0.33\textwidth}
        \centering
        \includegraphics[scale=0.4]{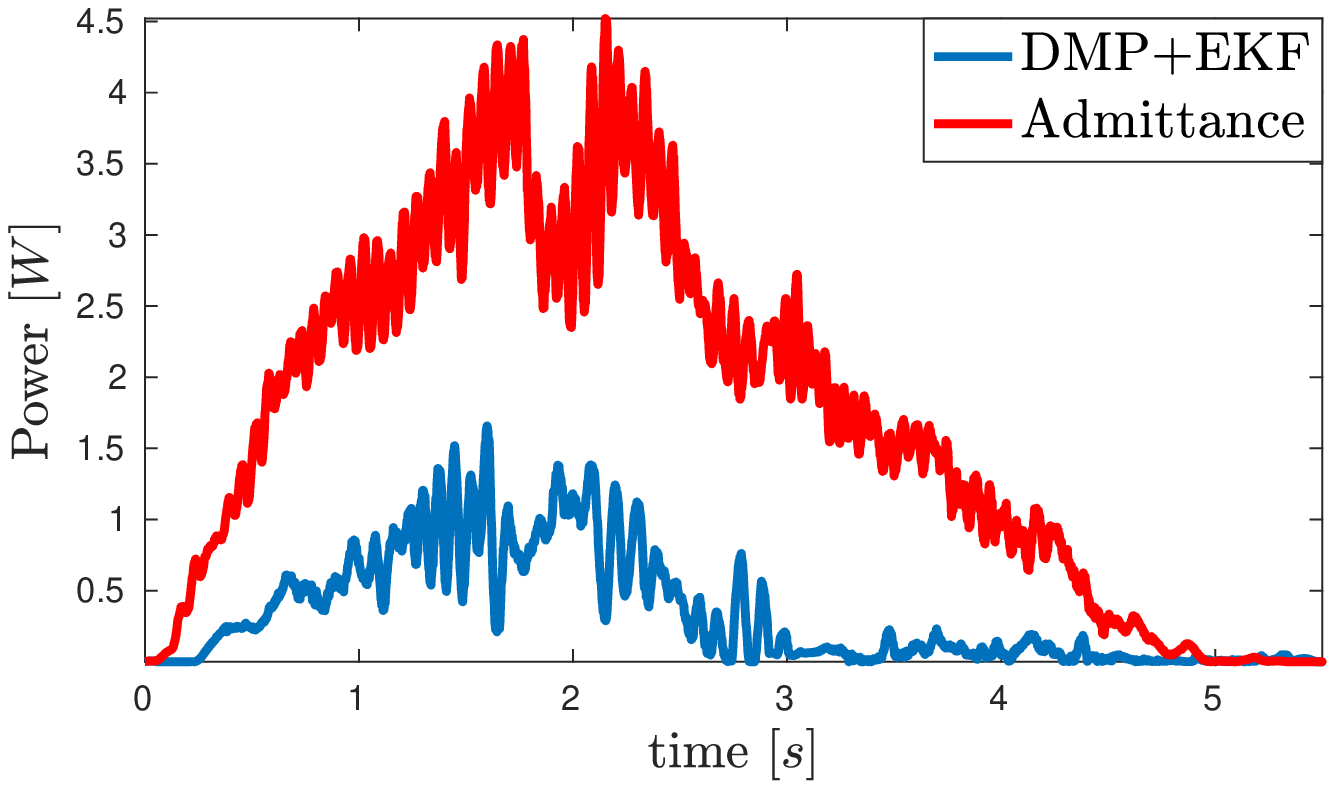}
        \caption{experiment 3}
    \end{subfigure}%
    \caption{Absolute power for the three first object transfer experiments. The blue graphs correspond to DMP+EKF and the red graphs to admittance.}
    \label{fig:exp3_power}
\end{figure*}

\begin{figure}[!ht]
	\centering
	\includegraphics[width=0.4\textwidth,height=0.52\textwidth]{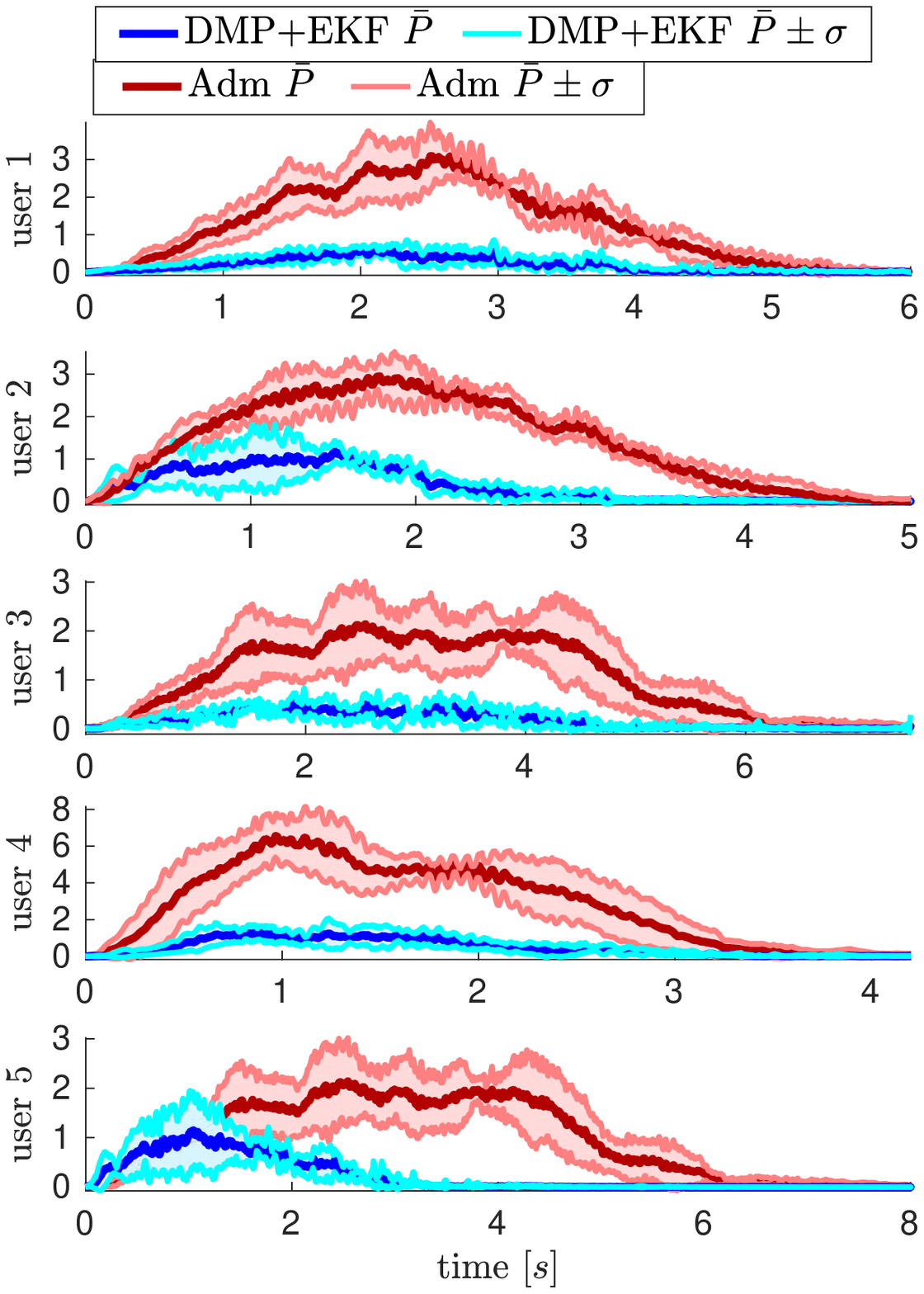}
	\caption{Power $\pm$ the standard deviation over five object transfer repetitions for each user. 
	The blue graphs correspond to using DMP+EKF and the red graphs to Admittance.}
	\label{fig:multi_power}
\end{figure}

\section{Conclusions} \label{sec:Conclusions}
In this work a DMP-based reference model and an EKF observer for predicting the target pose and time scaling for assisting the human proactively in the transportation of an object was proposed. The stability analysis that was carried out proves that the proposed scheme guarantees the boundedness of the reference model and observer. Experimental results validate the proposed approach, highlighting its practical benefits and efficiency with respect to human effort minimization. In this work the object's weight was assumed known a priori and only human exerted forces were considered. 
Our future work is oriented towards the extension of the proposed method to handle unknown object dynamics and the discrimination of human wrenches from the wrenches emerging due to contact with the environment.

\section*{Appendix A - Unit Quaternion Preliminaries}

Given a rotation matrix  $\vect{R}\in SO(3)$, an orientation can be expressed in terms of the unit quaternion $\vect{Q} \in \mathbb{S}^{3}$ as $\vect{Q}=[w \ \vect{v}^T]^T = [\cos(\theta) \ \sin(\theta) \vect{k}^T]$,
where $\vect{k} \in \mathbb{R}^3$, $2\theta \in [0 \ 2\pi)$ are the equivalent unit axis - angle representation.
The quaternion product between the unit quaternions $\vect{Q}_1$, $\vect{Q}_2$ is 
$\vect{Q}_1*\vect{Q}_2 =
    \begin{bmatrix}
    w_1 w_2 - \vect{v}_1^T\vect{v}_2 \\
    w_1\vect{v}_2 + w_2\vect{v}_1 + \vect{v}_1 \times \vect{v}_2
\end{bmatrix}$.
The inverse of a unit quaternion is equal to its conjugate which is
$\vect{Q}^{-1} = \bar{\vect{Q}} = [w \ -\vect{v}^T]^T$.
The quaternion logarithm is $\vect{q} = \log(\vect{Q})$, where $\log: \ \mathbb{S}^3 \rightarrow \mathbb{R}^3$ is defined as:
\begin{equation} \label{eq:quatLog}
    \log(\vect{Q}) \triangleq 
        \left\{
            \begin{matrix}
                2\cos^{-1}(w)\frac{\vect{v}}{||\vect{v}||}, \ |w|\ne1 \\
                [0,0,0]^T, \ \text{otherwise}
            \end{matrix}
        \right.
\end{equation}
The quaternion exponential is $\vect{Q} = \exp(\vect{q})$, where $\exp: \ \mathbb{R}^3 \rightarrow \mathbb{S}^3$ is defined as:
\begin{equation} \label{eq:quatExp}
    \exp(\vect{q}) \triangleq 
        \left\{
            \begin{matrix}
                [\cos(||\vect{q}/2||), \sin(||\vect{q}/2||)\frac{\vect{q}^T}{||\vect{q}||}]^T, \ ||\vect{q}|| \ne 0 \\
                [1,0,0,0]^T, \ \text{otherwise}
            \end{matrix}
        \right.
\end{equation}
If we limit the domain of the exponential map $\exp: \ \mathbb{R}^3 \rightarrow \mathbb{S}^3$ to $||\vect{v}|| < \pi$ and the domain of the logarithmic map to $\mathbb{S}^3/([-1, 0, 0, 0]^T)$, then both mappings become one-to-one, continuously differentiable and inverse to each other. 
The equations that relate the time derivative of $\vect{q}$ and the rotational velocity $\vect{\omega}$ and acceleration $\dvect{\omega}$ of $\vect{Q}$ are \cite{DMP_orient_Koutras}:
\begin{equation} \label{eq:dq_omega}
    \dvect{q} = \vect{J}_Q (\vect{\Omega} * \vect{Q})
\end{equation}
\begin{equation} \label{eq:omega_dq}
    \vect{\Omega} = 2 (\vect{J}_q\dvect{q})*\bvect{Q}
\end{equation}
\begin{equation} \label{eq:dvRot_ddq}
    \dvect{\Omega} = 2 (\dvect{J}_q\dvect{q} + \vect{J}_q\ddvect{q})*\bvect{Q} - \frac{1}{2}\begin{bmatrix} ||\vect{\omega}||^2 \\ \vect{0}_{3 \times 1}  \end{bmatrix}
\end{equation}
where $\vect{\Omega} = [0 \ \vect{\omega}^T]^T$ and
\begin{equation} \label{eq:J_q}
    \vect{J}_Q = 2
    \begin{bmatrix}
        \frac{\theta \cos(\theta) - \sin(\theta)}{\sin^2(\theta)}\vect{k} & \frac{\theta}{\sin(\theta)}\vect{I}_3
    \end{bmatrix}
\end{equation}
\begin{equation} \label{eq:J_Q}
    \vect{J}_q = \frac{1}{2}
    \begin{bmatrix}
        -\sin(\theta) \vect{k}^T \\
        \frac{\sin(\theta)}{\theta}(\vect{I}_3-\vect{k}\vect{k}^T) + \cos(\theta)\vect{k}\vect{k}^T
    \end{bmatrix}
\end{equation}


\section*{Appendix B - DMP Preliminaries}
\subsection{Cartesian position encoding}

A DMP for encoding a Cartesian position point-to-point motion can be expressed as \cite{Ijspeert2013}:
\begin{align} 
     &\ddvect{p} = \frac{1}{\tau^2}(\alpha_z \beta_z(\vect{p}_{g} - \vect{p}) - \alpha_z \tau \dvect{p} + g_f(x)\vect{K}_{p_g}\vect{f}(x)) \label{eq:DMP_ddp} \\
    &\tau \dot{x} = 1 \ , \ x(0) = 0 \label{eq:DMP_dx}
\end{align}
The desired motion is encoded by the forcing term, ${\vect{f}}(x) = \frac{\sum_{i=1}^{N} \vect{W}_i \psi_i(x)}{\sum_{i=1}^{N} \psi_i(x)}$ which is the weighted sum of $N$ Gaussian kernels,
with $\psi_i(x) = \exp(-h_i(x-c_i)^2)$. The matrix $\vect{W} \in \mathbf{R}^{3 \times N}$, with $\vect{W}_i$ denoting the $i_{th}$ column, contains in each row the weights for each Cartesian coordinate, which can be learned using Least Squares or Locally Weighted Regression (LWR) \cite{Schaal1998} based on the demonstrated data.
The DMP evolves based on the phase variable $x$, used to avoid direct time dependency. We use a linear canonical system \eqref{eq:DMP_dx} 
as in \cite{Antosidi_HRCOT_2019}.
The variable $\tau>0$ provides temporal scaling of the encoded motion pattern and the matrix $\vect{K}_{p_g} = diag(\vect{p}_g - \vect{p}_0)$ provides spatial scaling.
The sigmoid gating $g_f(x)= \frac{1}{1 + e^{a_g(x-c_g)}}$ ensures that the forcing term fades to zero at $x=1$ as in \cite{DMP_sigmoid_gating}, thus for $x \ge 1$ \eqref{eq:DMP_ddp}
acts as a pure spring-damper and converges asymptotically to the goal $\vect{p}_g$. 
Integrating \eqref{eq:DMP_dx} and substituting $x$ in \eqref{eq:DMP_ddp} we get \eqref{eq:DMP_ddpos}.




\subsection{Cartesian orientation encoding}
A formulation for orientation DMP that avoids undesired oscillations during reproduction as opposed to \cite{Gams_Ude_2014}  is proposed in \cite{DMP_orient_Koutras} and can be written as follows :
\begin{align}
     &\ddvect{q}' = \frac{1}{\tau^2}(\alpha_z \beta_z(\vect{q}'_{g} - \vect{q}') - \alpha_z \tau \dvect{q}' + g_f(x)\vect{K}_{q_g}\vect{f}(x)) \label{eq:DMP_ddq} \\
    &\tau \dot{x} = 1 \ , \ x(0) = 0 \label{eq:DMP_dx_o}
\end{align}
where, in \cite{DMP_orient_Koutras}, $\vect{q}' = \log(\vect{Q}_g*\bvect{Q})$ with $\vect{Q}$, $\vect{Q}_g$ the current and target orientation in unit quaternions respectively and $\vect{K}_{q_g}=diag(\log(\vect{Q}_g*\bvect{Q}_0))$. 
However, $\vect{q}'$ couples the DMP's state to $\vect{Q}_g$, which in our case is estimated online. Adopting this would incur extra dynamics to the DMP, affecting the execution and making the derivation of the proof in section \ref{sec:Stability_analysis} not possible. Instead, we retain the formulation given by \eqref{eq:DMP_ddq}, which exploits the use of the logarithmic map and define $\vect{q}' = \log(\vect{Q}*\bvect{Q}_0)$, where $\vect{Q}_0$ is the initial orientation. This is also in line with the position DMP, which is also anchored to the initial position, i.e. \eqref{eq:DMP_ddp} can be also written as $\ddvect{p}' = \frac{1}{\tau^2}(\alpha_z \beta_z(\vect{p}'_{g} - \vect{p}') - \alpha_z \tau \dvect{p}' + g_f(x)\vect{K}_{p_g}\vect{f}(x))$, where $\vect{p}' = \vect{p} - \vect{p}_0$ and $\vect{p}'_g = \vect{p}_g - \vect{p}_0$.
Integrating \eqref{eq:DMP_dx_o} and substituting $x$ in \eqref{eq:DMP_ddq} we get \eqref{eq:DMP_ddq_1}.

\section*{Appendix C - Proof of Theorem 1}

In the following we complete the proof of Theorem 1 from Section \ref{sec:Stability_analysis}.

\begin{proof}

The system given by \eqref{eq:model_ref_rewritten}, with $\vect{d}$ being a time varying bounded disturbance, can be decoupled in each dimension. Therefore we will conduct the analysis for one element of $\vect{q}'$ with dynamics:
\begin{equation} \label{eq:model_ref_rewritten_1d}
    \hat{\tau}_o^2\ddot{y} = 
    -\alpha_z \beta_z y 
    - \alpha_z \hat{\tau}_o \dot{y}
    + \hat{\tau}_o^2 d
\end{equation}
Introducing the state variable $\vect{\zeta} = [y \ \dot{y}]^T$ \eqref{eq:model_ref_rewritten_1d} can be written in matrix form:
    $
    \dot{\vect{\zeta}} = 
    \vect{A}(t) \vect{\zeta} 
    + \vect{B} d 
    $
where $\vect{A}(t) = \frac{1}{\hat{\tau}_o^2} \begin{bmatrix} 0 & 1 \\ -\alpha_z \beta_z & -\alpha_z \hat{\tau}_o\end{bmatrix}$ and $\vect{B} = \begin{bmatrix} 0 \\ 1 \end{bmatrix}$. 
Since $\alpha_z, \beta_z > 0$ and $\hat{\tau}_o$ is positive and bounded, $Re\{\lambda_i(\vect{A}(t))\} \le -\sigma_s$ $\forall t \ge 0 $, $i=1,2$, where $\sigma_s>0$ is constant. Taking into account that $\bar{\vect{C}}(t), \vect{P}(t) \in \vect{L}_{\infty}$ and $\vect{\tau}_{ext} \in \vect{L}_{\infty} \cap \vect{L}_2$ it follows from \eqref{eq:dtheta_hat} that $\dot{\hat{\tau}}_o \in \vect{L}_{\infty} \cap \vect{L}_2$. 
Moreover, $\dot{\vect{A}}(t) = \vect{A}_1(t) \dot{\hat{\tau}}_o$, where 
$\vect{A}_1(t) = \begin{bmatrix} 0 & -2/\hat{\tau}_o^3 \\ 2\alpha_z\beta_z/\hat{\tau}_o^3 & \alpha_z/\hat{\tau}_o^2 \end{bmatrix}$ and since $||\vect{A}_1(t)|| \in  \vect{L}_{\infty}$ and $\dot{\hat{\tau}}_o \in \vect{L}_{\infty} \cap \vect{L}_2$ we will also have that $||\dot{\vect{A}}(t)|| \le ||\vect{A}_1(t)|| |\dot{\hat{\tau}}_o| \in \vect{L}_{\infty} \cap \vect{L}_2 $.
 Finally, since $\vect{A}(t)$ is differentiable and bounded, based on Theorem $3.4.11$ from \cite{Robust_adapt_control_Ioannou}, the origin is uniformly globally asymptotically stable equilibrium for the system $\dot{\vect{\zeta}} = \vect{A}(t) \vect{\zeta}$. Therefore there exist matrices $\vect{\Pi}(t) = \vect{\Pi}^T(t) > 0$ and $\vect{Q}(t) = \vect{Q}^T(t) > 0$ with $\dot{\vect{\Pi}} = -\vect{A}^T(t)\vect{\Pi}(t) - \vect{\Pi}(t)\vect{A}(t) - \vect{Q}(t)$ satisfying $0 < \pi_1 < ||\vect{\Pi}(t)|| < \pi_2$ and $0 < q_1 < ||\vect{Q}(t)|| < q_2$ \cite{Marquez2003NonlinearCS}. Consequently the system 
 $
\dot{\vect{\zeta}} = 
\vect{A}(t) \vect{\zeta} 
+ \vect{B} d 
$
is uniformly ultimately bounded, which can be shown easily using the Lyapunov function $V = \vect{\zeta}^T \vect{\Pi}(t) \vect{\zeta}$. Hence, $y, \dot{y} \in \vect{L}_{\infty}$. The same analysis holds for each dimension, so $\dvect{q} \in \vect{L}_{\infty}$, therefore from \eqref{eq:omega_dq} we conclude that $\vect{\omega} \in \vect{L}_{\infty}$.
\end{proof}

\bibliographystyle{IEEEtran}
\bibliography{main} 

\end{document}